\documentclass{article}

\usepackage{microtype}
\usepackage{graphicx}
\usepackage{subfigure}
\usepackage{booktabs} 
\usepackage{float}

\usepackage{hyperref}

\usepackage[accepted]{icml2024}

\usepackage{amsmath}
\usepackage{amssymb}
\usepackage{mathtools}
\usepackage{amsthm}

\usepackage[capitalize,noabbrev]{cleveref}

\theoremstyle{plain}
\newtheorem{theorem}{Theorem}[section]

\theoremstyle{definition}

\theoremstyle{remark}

\usepackage[textsize=tiny]{todonotes}

\icmltitlerunning{Submission to ICML 2024}

\begin{document}

\twocolumn[
\icmltitle{Graph Multi-Similarity Learning for Molecular Property Prediction
}



\icmlsetsymbol{equal}{*}

\begin{icmlauthorlist}
\icmlauthor{Hao Xu}{equal,yyy}
\icmlauthor{Zhengyang Zhou}{equal,yyy}
\icmlauthor{Pengyu Hong}{yyy}

\end{icmlauthorlist}

\icmlaffiliation{yyy}{Department of Computer Science, Brandeis University, Waltham, USA}

\icmlcorrespondingauthor{Pengyu Hong}{hongpeng@brandeis.edu}

\icmlkeywords{Machine Learning, ICML}
\vskip 0.3in
]


\printAffiliationsAndNotice{\icmlEqualContribution} 
\begin{abstract}
Enhancing accurate molecular property prediction relies on effective and proficient representation learning. It is crucial to incorporate diverse molecular relationships characterized by multi-similarity (self-similarity and relative similarities) \cite{wang2019multi} between molecules. However, current molecular representation learning methods fall short in exploring multi-similarity and often underestimate the complexity of relationships between molecules. Additionally, previous multi-similarity approaches require the specification of positive and negative pairs to attribute distinct pre-defined weights to different relative similarities, which can introduce potential bias. In this work, we introduce Graph Multi-Similarity Learning for Molecular Property Prediction (GraphMSL) framework, along with a novel approach to formulate a generalized multi-similarity metric without the need to define positive and negative pairs. In each of the chemical modality spaces (e.g., molecular depiction image, fingerprint, NMR, and SMILES) under consideration, we first define a self-similarity metric (i.e., similarity between an anchor molecule and another molecule), and then transform it into a generalized multi-similarity metric for the anchor through a pair weighting function. GraphMSL validates the efficacy of the multi-similarity metric across MoleculeNet datasets. Furthermore, these metrics of all modalities are integrated into a multimodal multi-similarity metric, which showcases the potential to improve the performance. Moreover, the focus of the model can be redirected or customized by altering the fusion function. Last but not least, GraphMSL proves effective in drug discovery evaluations through post-hoc analyses of the learnt representations.
\end{abstract}

\section{Introduction}
\label{submission}


Graph Neural Networks (GNNs) have emerged as a prominent approach for molecular representation learning, addressing drug-related challenges \cite{wieder2020compact, zhang2022graph, fang2022geometry, wang2023motif}. However, the generation of task-specific labels for training molecular GNNs is hindered by the resource-intensive and time-consuming nature of chemical synthesis and biological testing experiments. To address this challenge, current research prioritizes self-supervised learning approaches for pretraining molecular GNNs, with a prevalent trend of adopting contrastive learning approaches \cite{wang2021molecular, wang2022molecular, liu2022attention}. (See an introduction of pre-training approaches in Section \ref{Preliminary:pre-training})

Contrastive Learning (CL) is a discriminative representation learning approach by bringing similar instances into close proximity within the latent representation space and pushing apart dissimilar instances \cite{schroff2015facenet}. A fundamental prerequisite of CL is to define positive pairs denoting similarity, and negative pairs representing dissimilarity \cite{jaiswal2020survey}. In the CL-based pre-training of molecular GNNs, positive pairs are often established through either \textit{data augmentation}  \cite{sun2021mocl, you2020graph}, such as node deletion, edge perturbation, subgraph extraction, attribute masking, and subgraph substitution, or \textit{domain knowledge}, exemplified by reactant-product pairing \cite{wang2022chemicalreactionaware} or conformer grouping \cite{moon20233d}. However, such a binary characterization of the relationships among molecules, by designating them as either positive or negative pairs, oversimplifies the complex nature of these connections. Moreover, these CL approaches fail to notice relationships among multiple instances simultaneously by adapting simple contrastive loss, thereby hindering the effectiveness and generalizability of representation learning \cite{wang2019multi,mu2023multi,zhang2023denoising}. (See an illustration of similarity types in CL in Figure Appendix \ref{fig:similarities})

\begin{figure*}[ht!]
    \centering
    \includegraphics[width=1\textwidth]{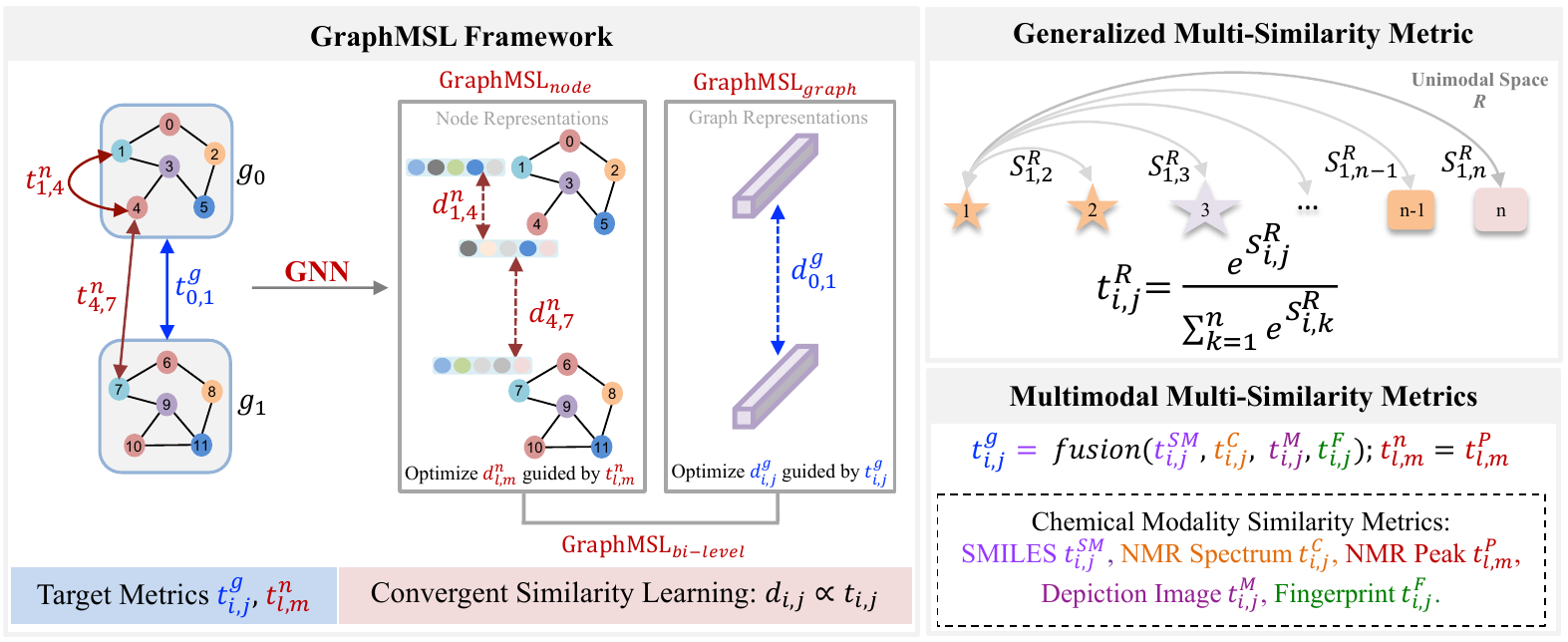}
    \caption{\textbf{Graph Similarity Learning for Molecular Property Prediction (GraphMSL).}  $\mathcal \mathrm t_{i,j}^{g}$, $\mathcal \mathrm t_{l,m}^{n}$ represent graph-level and node-level target similarity, respectively. $\mathcal \mathrm d_{i,j}^{g}$, $\mathcal \mathrm d_{l,m}^{n}$ represent the similarity for graph-level and node-level embeddings, respectively.  Here, $i$ and $j$ are the indices of molecular graphs in the graph pool, and $l$ and $m$ are the indices of atomic nodes in the node pool. Unlike the general contrastive learning framework shown in Appendix Figure \ref{fig:traditional-cl}, GraphMSL doesn't need to define positive or negative pairs and is capable of learning continuous ordering from target similarity.}
    \label{fig:main-structure}
\end{figure*}

Unlike CL, which adapts a binary similarity metric, Similarity Learning (SL) employs a continuous similarity metric for representation learning. It measures the similarity between two instances in the given space \cite{balcan2006theory, wen2023pairwise}. This pairwise similarity provides a localized perspective on the relations between two instances in a given instance pool, known as self-similarity \cite{wang2019multi}. However, the global relations between two instances can be influenced by their belonging to the instance pool. Therefore, self-similarity alone is inadequate for capturing relationships among multiple instances simultaneously, known as relative similarity \cite{wang2019multi}. To address this challenge, Multi-Similarity Learning (MSL) has emerged as a solution, expanding its focus from self-similarity to a global view of relations and encapsulating both self-similarity and relative similarity \cite{wang2019multi, zhang2021jointly, mu2023multi, zhang2023denoising}.


When formulating a similarity metric, a common guiding principle is that an effective similarity function or metric should align tightly with the objectives of specific tasks \cite{balcan2006theory, muller2018introduction,yang2006distance}. To target the drug discovery tasks, a robust similarity metric must be adept at discerning molecules based on key properties regarding drug development. Insights drawn from prior research \cite{xu2023asymmetric} indicate that distinct chemical modalities, such as chemical languages, molecular depiction images, and chemical spectroscopic spectra, possess unique expertise in expressing specific molecular properties. More importantly, a continuous similarity metric for molecules can be projected from each chemical modality space. Thus, a more promising multimodal similarity metric can be formulated by integrating these individual metrics.

In response to the challenges and opportunities in molecular graph representation learning, we propose the Graph Multi-Similarity Learning for Molecular Property Prediction (GraphMSL) framework. This approach aims to advance graph contrastive learning to graph multi-similarity learning by incorporating a continuous multi-similarity metric. A self-similarity metric can be derived from each heterogeneous chemical modality, such as chemical languages, molecular depiction images, and chemical spectroscopic spectra, through representation learning. Then, this self-similarity metric can be transformed into a multi-similarity metric through a pair weighting function. Besides, various unimodal multi-similarity metrics can be fused into a multimodal similarity metric. In addition, GraphMSL framework can be customized for singular or multi-view perspective. 


In summary, our contribution comprises three major aspects: 
\textbf{\textit{Conceptually:}} We introduce a generalized multi-similarity metric for graph representation learning, capturing both self-similarity and relative similarity. Our approach doesn't rely on pre-defined negative or positive pairs, and it satisfies the requirement of convergent similarity learning as shown in Section \ref{thm:convergent-smi-learning}. To the best of our knowledge, this is the first work to demonstrate such generalized multi-similarity for graph representation learning.
\textbf{\textit{Methodologically:}} 
We extract a self-similarity metric from a chemical modality, and transition it into a generalized multi-similarity metric through a pair weighting function, and each modality contributes to a unique multi-similarity metric. Furthermore, we integrate these metrics into a fused multimodal form that has the potential to improve the performance.
\textbf{\textit{Empirically:}} 
GraphMSL excels in various downstream tasks, with performance enhancements achieved through multi-level graph learning. Last but not least, we demonstrate the explainability of the learned representations through two post-hoc analysis. Notably, we explore minimum positive subgraphs and maximum common subgraphs to gain insights for further drug molecule design.


\section{Preliminaries}
\label{Preliminary:all}
\textbf{Directed Message Passing Neural Network (DMPNN).}
The Message Passing Neural Network (MPNN) \cite{gilmer2017neural} is a GNN model that processes an undirected graph $G$ with node (atom) features $x_v$ and edge (chemical bond) features $e_{vw}$. It operates through two distinct phases: a message passing phase, facilitating information transmission across the molecule to construct a neural representation, and a readout phase, utilizing the final representation to make predictions regarding properties of interest. The primary distinction between DMPNN and a generic MPNN lies in the message passing phase. While MPNN uses messages associated with nodes, DMPNN crucially differs by employing messages associated with directed edges \cite{yang2019analyzing}. This design choice is motivated by the necessity to prevent totters \cite{mahe2004extensions}, eliminating messages passed along paths of the form $v_1 v_2 \dots v_n$, where $v_i = v_{i+2}$ for some $i$, thereby eliminating unnecessary loops in the message passing trajectory. 

\textbf{Pre-Training for Molecular GNNs.} 
\label{Preliminary:pre-training}
There are two levels of pre-training tasks: node-level (atomic level) and graph-level (molecular level), which enhance the generalization capabilities of molecular GNNs across diverse downstream tasks \cite{hu2019strategies, xia2022mole, fang2022geometry}. Node-level tasks aim to capture local context, often involving the random masking of nodes and subsequent prediction of their properties based on node representations. In contrast, graph-level tasks focus on extracting global information, such as predicting graph properties based on the graph representation.Self-supervised learning (SSL) is a paradigm in which a model is trained on a task by leveraging the data itself to generate supervisory signals, eliminating the need for external annotations \cite{liu2021self, liu2022graph}. SSL finds active applications in the pre-training tasks of molecular GNNs by formulating label-free pretext tasks, such as graph reconstructions \cite{hu2020gpt, you2020does, liu2021pre}, context predictions \cite{hu2019strategies, peng2020self}, and adopting contrastive learning approaches \cite{wang2021molecular, wang2022molecular, liu2022attention}.

\textbf{Similarity Learning.}
Given object $i$ and object $j$, the optimization of object similarity $d_{i,j}$ in the latent representation space is directed by the target similarity $t_{i,j}$ in a given space, a process commonly recognized as similarity learning \cite{moutafis2016overview, suarez2018tutorial}. Contrastive learning constrains its similarity metric $t_{i,j}$ to a binary setting, taking on values of either 1 or 0. However, similarity learning allows its similarity metric $t_{i,j}$ to be continuous values. Two distinct types of similarities can be identified, as illustrated in Appendix Figure \ref{fig:similarities}: \textit{self-similarity} (the pairwise similarity between two objects, typically defined through cosine similarity), \textit{relative similarity} (distinctions in self-similarity with other pairs) \cite{wang2019multi}. Contrastive learning, implemented with a binary similarity metric, concentrates on self-similarity and fails to explore the complete relationships between samples \cite{oh2016deep,wang2019multi}.

\section{Methods}
We begin by presenting the theorem of convergent similarity learning, followed by introducing the generalized multi-similarity metric and the multimodal multi-similarity metric. 

\subsection{Convergent Similarity learning}
\label{thm:convergent-sl}

Let $\mathcal{S}$ be a set of instances with size of $|\mathcal{S}|$, and let $\mathcal{P}$ represent the learnable latent representations of instances in $\mathcal{S}$ such that $|\mathcal{P}| = |\mathcal{S}|$. For any two instances $i, j \in \mathcal{S}$, their respective latent representations are denoted by $\mathcal{P}_{i}$ and $\mathcal{P}_{j}$. Let $t_{i,j}$ represent the target similarity between instances $i$ and $j$ in a given domain, and let  $d_{i,j}$ be the similarity between $\mathcal{P}_i$ and $\mathcal{P}_j$ in the latent space.
\begin{theorem}[Theorem of Convergent Similarity learning]
\label{thm:convergent-smi-learning}
If $t_{i,j}$ is non-negative and $\{t_{i,j}\}$ satisfies the constraint $\sum_{j=1}^{|\mathcal{S}|}t_{i,j} = 1$, consider the loss function for an instance $i$ defined as follows:
\begin{equation}
    L(i) = -\sum_{j=1}^{|\mathcal{S}|} t_{i,j} \log \left( \frac{e^{d_{i,j}}}{\sum_{k=1}^{|\mathcal{S}|} e^{d_{i,k}}} \right)
\end{equation}
then when it reaches ideal optimum, the relationship between $t_{i,j}$ and $d_{i,j}$ satisfies:
\begin{equation}
    \text{softmax}(d_{i,j}) = t_{i,j}
\end{equation}
\end{theorem}
For detailed proof, please refer to Appendix Section \ref{appendix:gml-guide-proof}.


\subsection{Generalized Multi-Similarity}
We formulate a generalized multi-similarity metric from self-similarity by adapting the softmax function as a pair weighting function. The formula for the generalized multi-similarity, denoted as $t_{i,j}^{R}$ between the $i^{th}$ and $j^{th}$ instances under a given space $R$, is provided below:
\begin{equation}
{t}_{i,j}^{R} = softmax({S}_{i,j}^{R}) = \frac{e^{S_{i,j}^{R}}}{\sum_{k=1}^{|\mathcal{S}|} e^{{S}_{i,k}^{R}}}
\end{equation}
where $S_{i,j}^{R}$ represents self-similarity, and $|\mathcal{S}|$ is the size of the instance set. Defined in this manner, the generalized multi-similarity metric incorporates both self-similarity and relative similarities. Notably, unlike other multi-similarity learning approaches \cite{wang2019multi, zhang2021jointly}, our method does not rely on the categorization of negative and positive pairs for the pair weighting function. Additionally, the use of the softmax function ensures that the generalized target similarity $t_{i,j}$ adheres to the principles of Convergent Similarity Learning (refer to Section \ref{thm:convergent-sl}). 


\subsection{Multimodal Multi-Similarity}
With a set of generalized multi-similarities $\{t^{R}\}$ from various modality spaces, we can transform generalized multi-similarities from respective unimodality space to multimodal space through a fusion function. There are numerous potential designs of the fusion function. For simplicity, we take linear combination as a demonstration. The multimodal generalized multi-similarity $t_{i,j}^{M}$ between $i^{th}$ and $j^{th}$ objects can be defined as follows: 
\begin{align}
  t_{i,j}^{M} &= fusion(\{t^{R}\})\\
  &= \sum w_{R} \cdot t_{i,j}^{R}
\label{equ:graph-guide-zhou-general}
\end{align}
where $t_{i,j}^{R}$ represents the target similarity between $i^{th}$ and $j^{th}$ instance in unimodal space $R$, $w_{R}$ is the pre-defined weights for the corresponding modal, and $\sum w_{R} = 1$. Such that, it still satisfy the requirement of Convergent Similarity Learning (See proof in Appendix Section\ref{appendix:guarantee-fusion-sum}).

\section{Experiments}
In this section, we begin by presenting the design of the multi-similarity metric and the GraphMSL framework. Subsequently, we showcase the results obtained from GraphMSL. Finally, we demonstrate the explainability of the learned molecular representations. (Please refer to the experimental details of pre-training and fine-tuning in the Appendix Section \ref{appendix:exp-setting}.)

\subsection{The Design of Multi-Similarity Metric} 

\subsubsection{Graph-Level Multi-Similarity Metric}

\textbf{Self-Similarity.} By employing representation learning, various chemical modalities, such as NMR spectra, depiction images, and SMILES, can be encoded into latent representation vectors. The cosine similarity between two vectors serves as self-similarity. The unimodal self-similarity for $^{13}$C NMR spectrum, denoted as $S_{i,j}^{C}$, can be defined as follows:
\begin{equation}
    S_{i,j}^{C} = Cos( \mathcal V_{i}, \mathcal V_{j}) = \frac{\mathcal{V}_i \cdot \mathcal{V}_j^T}{\|\mathcal{V}_i\| \cdot \|\mathcal{V}_j\|}
\end{equation}
where $\mathcal V_{i},  \mathcal V_{j} $ represents the embedding of NMR spectra for two given molecules. Similarly, the uni-modality similarity similarity for depiction images and SMILES can be obtained. The self-similarity of fingerprints adapts a well-established similarity function for molecules, namely the Tanimoto similarity \cite{bajusz2015tanimoto}. (See more details in Appendix Section \ref{appendix:uni-modal-self})

\textbf{Generalized Multi-Similarity Metric.} Transitioning from self-similarity to a generalized similarity, we apply the softmax function as our pair weighting mechanism. The formula of generalized similarity is provided below, illustrated with an example of $^{13}$C NMR similarity:
\begin{equation}
t^{C}_{i,j} = softmax(S_{i,j}^{C})_i = \frac{e^{S_{i,j}^{C}}}{\sum_{k=1}^{|\mathcal{G}|} e^{s_{i,k}^{C}}}
\end{equation}
where $t^{C}_{i,j}$ represents generalized similarity, $S_{i,j}^{C}$ denotes the self-similarity, and $|\mathcal{G}|$ is the size of molecular graph pool. For the expressions for SMILES and Image modality, please refer to Appendix Section \ref{appendix:uni-modal-self}

\textbf{Multimodal Multi-Similarity Metric.} 
We employ a simple linear combination to formulate the multimodal multi-similarity $t_{i,j}^{M}$ between the $i^{th}$ and $j^{th}$ molecules, represented as a graph-level similarity $t_{i,j}^{g}$, as follows:
\begin{align}
t_{i,j}^{g} & = t_{i,j}^{M} \\
& = w_{SM} \cdot t^{SM}_{i,j} + w_{{C}} \cdot t^{C}_{i,j} + w_{{I}} \cdot t^{I}_{i,j} + w_{{F}} \cdot t^{F}_{i,j}
\label{equ:graph-guide-zhou-gen}
\end{align}
where $t^{SM}_{i,j}$ denotes the similarity based on SMILES, $t^{C}_{i,j}$ denotes the similarity with respect to $^{13}$C NMR spectrum, $t^{I}_{i,j}$ denotes the similarity regarding images, and $f$ denotes the similarity based on fingerprints, $w_{SM}$, $w_{C}$, $w_{I}$, and $w_{F}$ are the pre-defined weights for their respective similarity, and $w_{SM} + w_{{C}} + w_{{I}} + w_{{F}} = 1$. 

\subsubsection{Node-Level Multi-Similarity Metric}
\label{Knowledge-Span-ppm}
\textbf{Self-Similarity.} The self-similarity among nodes (atoms) is derived from the positions of their signal peaks on $^{13}$C NMR spectra, measured in parts per million (ppm). The ppm values are continuous, typically ranging from 0 to 200 (see more introduction of ppm in Appendix \ref{appendix:ppm}). The self-similarity of NMR peaks $S^{P}_{l,m}$ can be defined as following:
\begin{equation}
    S^{P}_{l,m} = \frac{\tau_{2}}{|ppm_{l} - ppm_{m}|+\tau_{1}}
\end{equation}
where $ppm_{l}$ and $ppm_{m}$ are the positions of NMR peaks for the $l^{th}$, $m^{th}$ Carbon atom, $\tau_{1}$ and $\tau_{2}$ are temperature hyper-parameter.

\textbf{Generalized Similarity Metric.} A generalized multi-similarity $t^{P}_{l,m}$, as a node-level similarity $t_{l,m}^{n}$, can be formulated with softmax function, as shown below
\begin{equation}
t_{l,m}^{n} = t^{P}_{l,m}
= softmax(S^{P}_{l,m})_l = \frac{e^{S^{P}_{l,m}}}{\sum_{q=1}^{|\mathcal{N}|} e^{S^{P}_{l,q}}}
\end{equation}
where $S^{P}_{l,m}$ represents self-similarity of NMR peaks, $|\mathcal{N}|$ is the size of atomic node pool.

\subsection{The Design of GraphMSL Framework} 

The GraphMSL framework is versatile, allowing customization for singular or multiple views. In this context, we showcase the customization options for  graph-level, node-level, or bi-level GraphMSL. The graph-level GraphMSL, denoted as $\text{GraphMSL}_{graph}$, incorporates multi-similarity metrics derived from molecular level chemical semantics. The node-level GraphMSL, denoted as $\text{GraphMSL}_{node}$, incorporates multi-similarity metrics derived from atomic level chemical semantics. The bi-level  GraphMSL, denoted as $\text{GraphMSL}_{bi-level}$, incorporates multi-similarity metrics derived from both molecular and atomic level chemical semantics. In these implementations, the graph encoder adopts DMPNN \cite{yang2019analyzing} architecture, an interactive message passing scheme considering the interactions. Additionally, neither of these two modules requires additional projection layers.
 
The loss function of $\text{GraphMSL}_{graph}$ model, noted as $L_{graph}$, can be expressed as:
\begin{equation}
    L_{graph} 
    =-\frac{1}{|\mathcal{G}|}\sum_{\substack{1\leq j \leq |\mathcal{G}|}} \mathcal \mathrm t_{i,j}^{g} \log \frac{e^{\mathcal \mathrm{d}_{i,j}^{g}}}{\sum_{\substack{1\leq k \leq |\mathcal{G}|}} e^{\mathcal \mathrm{d}_{i,k}^{g}}}
\end{equation}
where $t_{i,j}^{g}$ represents graph-level similarity metrics, $d_{i,j}^{g}$ represents graph-level latent space similarity metrics, $i, j, k$ represent the indices of molecular graphs within a graph pool of size $|\mathcal{G}|$ per batch. 

The loss function of $\text{GraphMSL}_{node}$ model, noted as $L_{node}$, can be expressed as:
\begin{equation}
    L_{node} 
    =-\frac{1}{|\mathcal{N}|}\sum_{\substack{1\leq m \leq |\mathcal{N}|}} \mathcal \mathrm t_{l,m}^{n} \log \frac{e^{\mathcal \mathrm{d}_{l,m}^{n}}}{\sum_{\substack{1\leq q \leq |\mathcal{N}|}} e^{\mathcal \mathrm{d}_{lq}^{n}}}
\end{equation}
where $t_{l,m}^{n}$ represents node-level similarity metrics, $d_{l,m}^{g}$ represents node-level latent space similarity metrics, $l, m, q$ represent the indices of nodes within a node pool of size $|\mathcal{N}|$.

The loss function of $\text{GraphMSL}_{bi-level}$ model, noted as $L_{bi-level}$, can be expressed as:
\begin{equation}
    L_{bi-level} = L_{graph} + L_{node}
\end{equation}



\subsection{Results}
\subsubsection{Overall performance of $\text{GraphMSL}_{{graph}}$}

The performance of $\text{GraphMSL}_{graph}$ is evaluated through a comparative analysis with a range of baselines, the specifics of which are described in Appendix Section \ref{appendix:baselines}. We report the performance metrics of $\text{GraphMSL}_{graph}$ across 8 classification under ROC-AUC and 3 regression tasks under RMSE from the MoleculeNet benchmark \cite{wu2018moleculenet}, as shown in Tables \ref{table:overall_performance}. Within these tables, the best results are denoted in bold, and the second-best are indicated with underlining. From the comparative evaluation, we find that: 

1) $\text{GraphMSL}_{graph}$ outperforms the baselines in seven of the eight evaluated classification tasks, including BBBP, BACE, SIDER, HIV, MUV, Tox21 and ToxCast. 

2) In the regression tasks, $\text{GraphMSL}_{graph}$ also achieves the highest performance across all evaluated benchmarks, which include ESOL, FreeSolv, and Lipophilicity.

3) While $\text{GraphMSL}_{graph}$ doesn't outperform baselines on the Clintox task, it attains average performance.

In short, these findings underscore the proficiency of $\text{GraphMSL}_{graph}$ in learning molecular representations that are  effective and impactful, as evidenced by its commendable performance across a suite of diverse tasks.

\setlength{\tabcolsep}{6pt}
\begin{table*}[ht!]
\caption{Overall performances (ROC-AUC) on classification downstream tasks. The best results are denoted in bold, and the second-best are indicated with underlining. (Note: N-Gram is highly time-consuming on ToxCast.) }
\vspace{2pt}
\label{table:overall_performance}
\begin{center}
\begin{scriptsize}
\begin{sc}
\begin{tabular}{lcccccccc}
\toprule
Data Set & BBBP & bace & Sider & Clintox & HIV & MUV & Tox21 & ToxCast \\
\midrule
AttentiveFP & 64.3$\pm$1.8 & 78.4$\pm$2.2  & 60.6$\pm$3.2 & 84.7$\pm$0.3 & 75.7$\pm$1.4 & 76.6$\pm$1.5 & 76.1$\pm$0.5 & 63.7$\pm$0.2  \\


DMPNN & 91.9$\pm$3.0 & 85.2$\pm$0.6 & 57.0$\pm$0.7 & 90.6$\pm$0.6 & 77.1$\pm$0.5 & 78.6$\pm$1.4 & 75.9$\pm$0.7 & 63.7$\pm$0.2  \\


\text{N-Gram} & 91.2$\pm$0.3 & 79.1$\pm$1.3 & 63.2$\pm$0.5 & 87.5$\pm$2.7 & 78.7$\pm$0.4  & 76.9$\pm$0.7 & 76.9$\pm$2.7 & -\\


GEM & 72.4$\pm$0.4 & 85.6$\pm$1.1 & 67.2$\pm$0.4 & 90.1$\pm$1.3 & 80.6$\pm$0.9 & 81.7$\pm$0.5 & 78.1$\pm$0.1 & 69.2$\pm$0.4\\

Uni-Mol & 72.9$\pm$0.6 & 85.7$\pm$0.2 & 65.9$\pm$1.3 & \underline{91.9$\pm$1.8} & 80.8$\pm$0.3 & \underline{82.1$\pm$1.3} & 79.6$\pm$0.5 & 69.6$\pm$0.1 \\

GROVER & 86.8$\pm$2.2 & 82.4$\pm$3.6 & 61.2$\pm$2.5 & 70.3$\pm$13.7 & 68.2$\pm$1.1 & 67.3$\pm$1.8 & 80.3$\pm$2.0 & 56.8$\pm$3.4\\

InfoGraph & 69.2$\pm$0.8 & 73.9$\pm$2.5 & 59.2$\pm$0.2 & 75.1$\pm$5.0 & 74.5$\pm$1.8 &  74.0$\pm$1.5 & 73.0$\pm$0.7 & 62.0$\pm$0.3 \\

GraphCL & 67.5$\pm$3.3 & 68.7$\pm$7.8 & 60.1$\pm$1.3 & 78.9$\pm$4.2 & 75.0$\pm$0.4 & 77.1$\pm$1.0 & 75.0$\pm$0.3 & 62.8$\pm$0.2 \\


MolCLR & 73.3$\pm$1.0 & 82.8$\pm$0.7 & 61.2$\pm$3.6 & 89.8$\pm$2.7 & 77.4$\pm$0.6 & 78.9$\pm$2.3 & 74.1$\pm$5.3 & 65.9$\pm$2.1 \\

$\text{MolCLR}_{\text{cmpnn}}$ & 72.4$\pm$0.7 & 85.0$\pm$2.4 & 59.7$\pm$3.4 & 88.0$\pm$4.0 & 77.8$\pm$5.5 & 74.5$\pm$2.1 & 78.4$\pm$2.6 & 69.1$\pm$1.2 \\

GraphMVP & 72.4$\pm$1.6 & 81.2$\pm$9.0 & 63.9$\pm$1.2 & 79.1$\pm$2.8 & 77.0$\pm$1.2 & 77.7$\pm$6.0 & 75.9$\pm$5.0 & 63.1$\pm$0.4\\

\hline

$\text{GraphMSL}_{graph}$ & 93.2$\pm$0.8 & \underline{93.6$\pm$2.7} & \textbf{68.1$\pm$1.5} & 88.8$\pm$4.6 & \textbf{83.3$\pm$1.1}  & \textbf{84.5$\pm$2.9} & \underline{86.1$\pm$0.6} & \textbf{71.4$\pm$0.2}\\
$\text{GraphMSL}_{node}$& \underline{93.4$\pm$2.7} & 89.3$\pm$1.7 & 62.8$\pm$2.1 & 86.1$\pm$5.4& 82.1$\pm$0.4 &75.4$\pm$5.2 &84.9$\pm$1.0 & 70.6$\pm$0.8 \\
$\text{GraphMSL}_{bi-level}$& \textbf{94.3$\pm$0.8} & \textbf{94.5$\pm$0.7} & \underline{67.3$\pm$0.6} & \textbf{93.8$\pm$0.8} & \underline{83.0$\pm$0.7} &81.5$\pm$3.7 & \textbf{86.1$\pm$0.8} & \underline{71.2$\pm$1.1}  \\

\bottomrule
\end{tabular}
\end{sc}
\end{scriptsize}
\end{center}
\end{table*}

\begin{table}[t]
\setlength{\tabcolsep}{6pt}
\caption{Overall performances (RMSE) on regression downstream tasks. The best results are denoted in bold, and the second-best are indicated with underlining.}
\label{table:overall_performance_regression}
\begin{center}
\begin{scriptsize}
\begin{tabular}{lccc}
\toprule
Data Set & ESOL & FreeSolv & Lipo  \\
\midrule
AttentiveFP & 0.877$\pm$0.029 & 2.073$\pm$0.183 & 0.721$\pm$0.001 \\

DMPNN & 1.050$\pm$0.008 & 2.082$\pm$0.082 & 0.683$\pm$0.016\\

$\text{N-Gram}_{\text{RF}}$  & 1.074$\pm$0.107 & 2.688$\pm$0.085 & 0.812$\pm$0.028 \\

$\text{N-Gram}_{\text{XGB}}$ & 1.083$\pm$0.082 & 5.061$\pm$0.744 & 2.072$\pm$0.030   \\

GEM & 0.798$\pm$0.029 & 1.877$\pm$0.094 & 0.660$\pm$0.008 \\

Uni-Mol & \underline{0.788$\pm$0.029} & 1.620$\pm$0.035 & 0.660$\pm$0.008 \\




GROVER & 1.423$\pm$0.288 & 2.977$\pm$0.615 & 0.823$\pm$0.010\\

MolCLR & 1.113$\pm$0.023 & 2.301$\pm$0.247 & 0.789$\pm$0.009  \\

$\text{MolCLR}_{\text{CMPNN}}$  & 0.911$\pm$0.082 & 2.021$\pm$0.133 & 0.875$\pm$0.003 \\


\hline
$\text{GraphMSL}_{graph}$ & \textbf{0.746$\pm$0.060} & \textbf{1.437$\pm$ 0.134} & \textbf{0.537$\pm$0.005} \\
$\text{GraphMSL}_{node}$ & 0.924$\pm$0.083 &1.707$\pm$0.126 & 0.587$\pm$0.021  \\
$\text{GraphMSL}_{bi-level}$ & 0.843$\pm$0.094 & \underline{1.601$\pm$0.057} & \underline{0.562$\pm$0.005} \\
\bottomrule
\end{tabular}
\vspace{-8pt} 
\end{scriptsize}
\end{center}
\end{table}

\subsubsection{Ablation Study-Various Mutil,modal Similarity Metrics}
\label{MetricsStudy}

We evaluate the performance of $\text{GraphMSL}_{graph}$ using a diverse set of similarity metrics, as outlined in Table \ref{tab:multimodal-metrics-single-level}. Each uni-modal similarity metric demonstrates unique strengths across various tasks. For instance, the model guided by Image similarity metrics exhibits outstanding performance in ESOL compared to other uni-modality metrics. Solubility is closely tied to the polarity of molecules. High polar atoms possess the ability to form hydrogen bonds with water, thereby enhancing solubility. In molecular depiction images, nonpolar carbon atoms (C) are typically represented as dots. Conversely, highly polar atoms like oxygen (O), nitrogen (N), and fluorine (F) are depicted more prominently, occupying a significant portion of the image. Consequently, image representations place considerable emphasis on this information, distinguishing between nonpolar and polar atoms by varying pixel density. 

Notably, the true strength lies in the flexibility of multimodal similarity metrics achieved through the fusion of multiple unimodal  metrics. There are five variations of multimodal similarity metrics, denoted as  $\text{Fusion}_\text{{Smiles}}$, $\text{Fusion}_\text{{NMR}}$, $\text{Fusion}_\text{{Image}}$, $\text{Fusion}_\text{{Fingerprint}}$, and $\text{Fusion}_\text{{Average}}$ (Please refer to their configurations in Appendix Section \ref{appendix:config-fusion}). Through the comparisons, it becomes evident that a well-designed multimodal similarity metric can significantly enhance model performance compared to unimodal metrics. For example, $\text{Fusion}_{\text{Smiles}}$ boosts performance in Tox21 tasks, $\text{Fusion}_{\text{NMR}}$ enhances results in MUV task, and $\text{Fusion}_{\text{Fingerprint}}$ contributes to improved outcomes in BBBP and BACE tasks.

\setlength{\tabcolsep}{4pt}


\subsubsection{Ablation Study-Mutil-View Learning}

While the $\text{GraphMSL}_{{graph}}$ model demonstrates its efficacy across a variety of downstream tasks, we further investigate the potential of GraphMSL framework by including both graph (molecule) and node (atom) levels to assess possible performance improvements. Our findings indicate that:

1) For Clintox, $\text{GraphMSL}_{{bi-level}}$ exhibits superior performance, compared with the baselines and $\text{GraphMSL}_{{graph}}$. Notably, $\text{GraphMSL}_{{bi-level}}$ achieves a 5\% enhancement in ROC-AUC on the Clintox dataset over $\text{GraphMSL}_{{graph}}$.

2) For tasks such as BBBP, BACE, HIV, MUV, and TOX21, $\text{GraphMSL}_{{bi-level}}$ outperforms all compared models, though $\text{GraphMSL}_{{bi-level}}$ is better than $\text{GraphMSL}_{{graph}}$ by a marginal degree.

3) Conversely, $\text{GraphMSL}_{{bi-level}}$ underperforms relative to $\text{GraphMSL}_{{graph}}$ in Sider, ToxCast, ESOL, FreeSolv, and Lipo datasets. This discrepancy in performance may stem from  small portion of unresolved interactions or slight discord between the graph-level and node-level similarities under the scenarios of these tasks.

\subsection{Explainability of Learnt Representations}
To demonstrate the interpretability of learnt representations, we present post-hoc analysis for two tasks, ESOL and BACE, as demonstration. The results showcase learnt representations can capture task-specific patterns and offer valuable insights for molecular design.

\textbf{ESOL}. We apply t-SNE to reduce molecule embeddings from 300 to 2, generating a heatmap correlating Log solubility (Figure 3.a). The heatmap visually depicts a smooth transition from high solubility (depicted in red) to low solubility (depicted in blue). The embeddings adeptly capture essential structural information and solubility-related patterns, organizing molecules with analogous solubility in close proximity within the embedding space. Through our investigation, we discovered that molecules with comparable solubility share common graph features, such as recurring motifs or an increased frequency of specific nodes (atoms). For example, in the region of lowest solubility (enclosed by blue dashed line), the molecular graphs all contain biphenyl groups. Biphenyls are nonpolar molecules, and as a consequence, they have limited interaction with water, leading to low solubility in aqueous environments. In the region of highest solubility (enclosed by red solid line), molecules exhibit high polarity, characterized by the prevalence of nitrogen (N) and oxygen (O) atoms. This specific graph configuration facilitates the formation of hydrogen bonds with water molecules, thereby enhancing solubility.
\setlength{\tabcolsep}{2.5pt}
\begin{table*}[ht]
\caption{Ablation study on the performances of $\text{GraphMSL}_{graph}$. The best results are denoted in boldf, and the second-best are indicated with underlining. The first 8 tasks are for classification under evaluation of ROC-AUC, while the last three are for regression with evaluation of RMSE. See the detailed performance of $\text{GraphMSL}_{bi-level}$ upon different similarity metrics in Appendix Table \ref{tab:multimodal-metrics-bi-level}. }
\label{tab:multimodal-metrics-single-level}
\begin{center}
\begin{scriptsize}
\begin{sc}
\begin{tabular}{lcccccccc|ccc}
\toprule
Data Set & BBBP & bace & Sider & Clintox & HIV & MUV & Tox21 & ToxCast & ESOL & FreeSolv & Lipo \\
\midrule
SMILES & 92.9$\pm$1.5 & 90.9$\pm$3.3 & 64.9$\pm$0.3 & 78.2$\pm$1.9  & \textbf{83.3$\pm$1.1} & 80.1$\pm$2.5 & 85.7$\pm$1.2 &70.5$\pm$2.5 & 0.811$\pm$ 0.109 & 1.623$\pm$ 0.168 & 0.539$\pm$ 0.017\\

NMR & 91.0$\pm$2.0 & 93.2$\pm$2.7 & \textbf{68.1$\pm$1.5} & 87.7$\pm$6.5  & 80.9$\pm$5.0 & 80.9$\pm$5.0 &  85.1$\pm$0.4 & 71.1$\pm$0.8 & 0.844$\pm$ 0.123 & 2.417$\pm$ 0.495 & 0.609$\pm$ 0.031 \\

Image & 93.1$\pm$2.4 & 92.9$\pm$1.8 & 65.3$\pm$1.5 & 86.2$\pm$6.5 & 82.3$\pm$0.6 & 78.7$\pm$1.7 & \underline{86.0$\pm$1.0} &71.0$\pm$1.6& \underline{0.761$\pm$ 0.068} & 1.648$\pm$ 0.045 & \textbf{0.537$\pm$ 0.005} \\

Fingerprint & 92.9$\pm$2.3 & 91.7$\pm$3.6  & 65.6$\pm$0.7 & 87.5$\pm$6.0 & 81.2$\pm$2.5 & \underline{82.9$\pm$3.1} & 85.3$\pm$1.3 &70.0$\pm$1.4 &0.808$\pm$ 0.071 & \textbf{1.437$\pm$ 0.134} & 0.565$\pm$ 0.017\\
\hline


$\text{Fusion}_{\text{Smiles}}$ & \underline{93.1$\pm$1.4} & 91.4$\pm$3.9 & 66.1$\pm$1.0 & 86.6$\pm$6.6 & 82.7$\pm$1.1 & 82.2$\pm$4.1 &  \textbf{86.1$\pm$0.6} & \underline{71.3$\pm$1.3} &0.800$\pm$0.068 & 1.505$\pm$0.177 & \underline{0.537$\pm$0.145} \\

$\text{Fusion}_{\text{NMR}}$ & 93.0$\pm$1.6 & 93.0$\pm$2.4 & 64.3$\pm$1.9 & 83.5$\pm$10.6 &  81.4$\pm$3.1& \textbf{84.5$\pm$2.9} & 85.8$\pm$1.1 & 70.9$\pm$1.1 &0.783$\pm$0.105 & \underline{1.472$\pm$0.072} & 0.552$\pm$0.029\\

$\text{Fusion}_{\text{Image}}$  & 92.9$\pm$3.4 & 92.9$\pm$2.4 & 64.3$\pm$1.6 & \underline{88.6$\pm$4.6} & \underline{83.0$\pm$0.9} & 81.6$\pm$4.8 & 85.8$\pm$3.8 & 70.6$\pm$1.7 &\textbf{0.746$\pm$0.060} & 1.587 $\pm$0.143 & 0.549$\pm$0.025 \\

$\text{Fusion}_{\text{Fingerprint}}$ & \textbf{93.2$\pm$0.8} & \textbf{93.6$\pm$2.7} & 65.8$\pm$0.7 & 85.4$\pm$9.4 & 82.4$\pm$3.1 & 81.6$\pm$2.5 & 85.3$\pm$1.1 & 71.1$\pm$1.1&0.818$\pm$0.054 & 1.535$\pm$0.080 & 0.573$\pm$0.040 \\

$\text{Fusion}_{\text{Average}}$  & 90.2$\pm$6.8 & \underline{93.4$\pm$2.7} & \underline{67.0$\pm$0.6} & \textbf{88.8$\pm$4.6} & 80.8$\pm$2.2 & 79.2$\pm$5.4 & 85.4$\pm$0.8 & \textbf{71.4$\pm$0.2} &0.781$\pm$0.082 & 1.528$\pm$0.180 & 0.559$\pm$0.018\\



\bottomrule
\end{tabular}
\end{sc}
\end{scriptsize}
\end{center}
\end{table*}

\begin{figure*}[ht]
    \centering
    \includegraphics[width=0.95\textwidth]{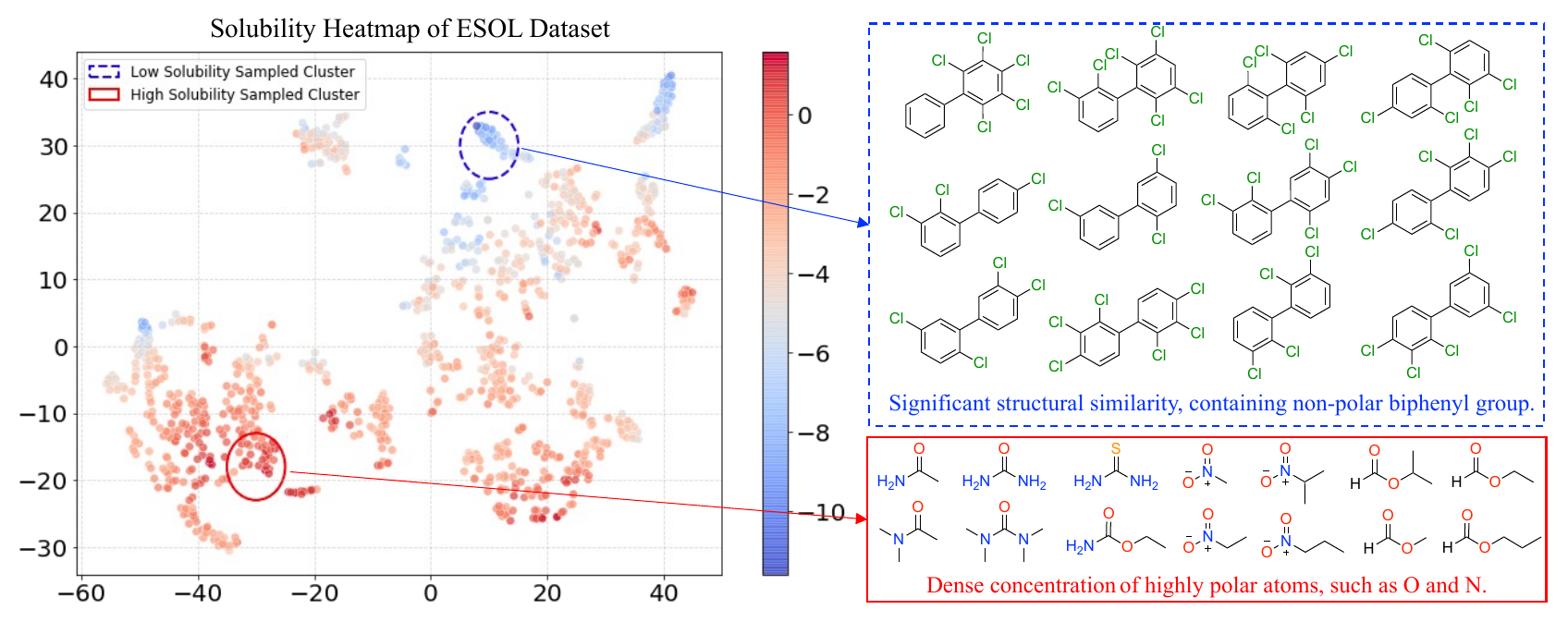}
    \caption{T-SNE visualization depicting the ESOL molecule embeddings alongside molecules within the highlighted region. Each point in the heatmap corresponds to the embeddings of respective molecules in ESOL, with color indicating solubility levels. Red denotes higher solubility, while blue indicates lower solubility.}
    \label{fig:ESOL-visualization}
\end{figure*}

\begin{figure*}[ht]
    \centering
    \includegraphics[width=0.95\textwidth]{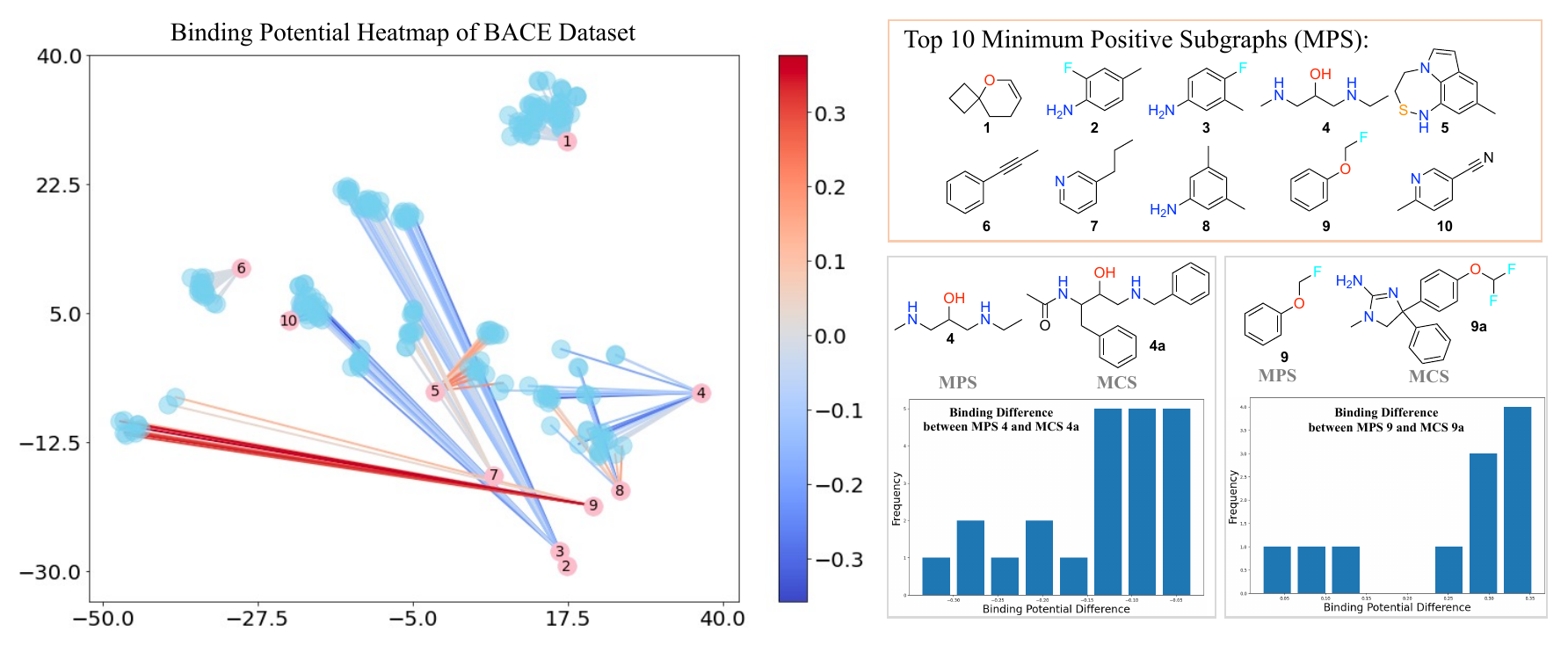}
    \caption{T-SNE Visualization of BACE embedding and clustering based on minimum positive subgraph (MPS). MPS represents minimum positive subgraph of a positive molecule; MCS represents maximum common subgraph of several positive molecules, sharing the same MPS. Pink nodes represent MPS, blue nodes depict molecules, and edge colors indicate binding potential differences. Red edges denote successful designs (original higher than MPS), while blue indicates less efficient designs (original lower than MPS).}
    \label{fig:bace-visualization}
\end{figure*}

\textbf{BACE}. We explore the binding potential of positive inhibitor molecules targeting BACE and their associated key functional substructures, referred to as minimum positive subgraphs (MPS). To identify MPS, we employ a Monte Carlo Tree Search (MCTS) approach integrated into our BACE classification model, as implemented in RationalRL \cite{jin2020multiobjective}. MCTS, being an iterative process, allows us to evaluate each candidate substructure for its binding potential with our model. Following the determination of MPSs, we categorize the original positive BACE molecules based on their respective MPSs and subsequently compute the maximum common subgraph (MCS) for each group. In Figure 3.b, we present the top 10 most frequently occurring MPSs, accompanied by their respective groups. As expected, molecular graph embeddings within the same group exhibit proximity after T-SNE reduction. This proximity is a result of shared identical motifs among graphs within each group.

In the development of new inhibitors, it is common to introduce additional functional groups or complex motifs to have more interactions with enzyme, thereby increasing the complexity of the molecules and making synthesis more challenging. However, such modifications do not guarantee a positive contribution or reward to inhibitor design. It is evident by the results of extending molecules from MPS to original molecules, which demonstrate both positive and negative contributions to binding potential. \textbf{Negative Contribution:} The complex molecules exhibits lower binding potentials than simple MPS \textbf{4}. This disparity can be illustrated by the distinction between MPS \textbf{4} and MCS \textbf{4a}: in MPS \textbf{4}, one of the NH groups, as a secondary amine, establishes a strong binding interaction with the enzyme, but this NH group transforms into an amide group in MCS \textbf{4a}, resulting in a significant decrease in binding capability. Therefore, these designs do not receive commensurate rewards considering the increased difficulty in synthesis. \textbf{Positive Contribution:} the complex molecules demonstrates higher binding potentials than MPS \textbf{9}, indicating successful design. This outcome can be attributed to MCS \textbf{9a}, which features more nitrogen and fluorine atoms capable of forming strong bindings with the enzyme. In summary, this experiment can serve as valuable guidelines for advancing inhibitor development.

\section{Related work}

\textbf{Contrastive Learning on Molecular Graphs.} The primary focus within the domain of contrastive learning applied to molecular graphs centers on 2D-2D graphs comparisons. Noteworthy representative examples: InfoGraph \cite{sun2019infograph} maximizes the mutual information between the representations of the graph and its substructures to guide the molecular representation learning;
GraphCL \cite{you2020graph}, MoCL \cite{sun2021mocl}, and MolCLR \cite{wang2022molecular} employs graph augmentation techniques to construct positive pairs; MoLR \cite{wang2022chemicalreactionaware}  establishes positive pairs with reactant-product relationships. In addition to 2D-2D graph contrastive learning, there are also noteworthy efforts exploring 2D-3D and 3D-3D contrastive learning in the field.
3DGCL \cite{moon20233d} is 3D-3D contrastive learning model, establishing positive pairs with conformers from the same molecules. GraphMVP \cite{liu2022pretraining}, GeomGCL \cite{li2022geomgcl}, and 3D Informax \cite{stark20223d} proposes 2D–3D view contrastive learning approaches. To conclude, 2D-2D and 3D-3D comparisons are intra-modality contratsive leraning, as only one graph encoder is employed in these studies. And these approaches often focus on  the motif and graph levels, leaving atom-level contrastive learning less explored.

\textbf{Multi-Similarity Learning.} Instance-wise discrimination, a crucial facet of similarity learning, involves evaluating the similarity between instances directly based on their latent representations or features \cite{wu2018unsupervised}. Naive instance-wise discrimination relies on self-similarity, leading to the development of contrastive loss \cite{hadsell2006dimensionality}. Although there are improved loss functions such as triplet loss \cite{hoffer2015deep}, quadruplet loss \cite{law2013quadruplet}, lifted structure loss \cite{oh2016deep}, N-pairs loss \cite{sohn2016improved}, and angular loss \cite{wang2017deep}, these methods still fall short in thoroughly capturing relative similarities \cite{wang2019multi}. To address this limitation, a joint multi-similarity loss has been proposed, incorporating pair weighting for each pair to enhance instance-wise discrimination \cite{wang2019multi, zhang2021jointly}. Notably, it is crucial to emphasize that employing these pair weightings requires the manual categorization of negative and positive pairs, as distinct weights are assigned to losses based on their categories.

\section{Discussion}

In summary, unlike other multi-similarity learning approaches that require explicit categorization of negative and positive pairs, our method enables a straightforward generalization of similarity measures, encapsulating both self-similarities and relative-similarities. Meanwhile, the generalized multi-similarity metrics satisfy the requirement of convergent similarity learning. Notably, our model adeptly integrates chemical semantics from diverse modalities, enhancing its performance across various downstream tasks. 
Additionally, our framework bridges machine learning and chemical domain knowledge through post-hoc experiment by identifying easily synthesizable and functional substructures, which can be refined into appropriate configurations by experts. Despite these accomplishments, further exploration is needed to achieve more effective integration of graph- and node-level similarities. Looking ahead, we are enthusiastic about the prospect of applying our model to additional fields, such as social science, thereby broadening its applicability and impact.

\section*{Accessibility}
The code and dataset will be made available upon the date of publication.



\bibliography{example_paper}
\bibliographystyle{icml2024}

\newpage
\appendix
\onecolumn
\begin{center}
   \Large{\textbf{Appendix}}
\end{center}
\counterwithin{figure}{section}
\counterwithin{equation}{section}
\counterwithin{table}{section}
\setcounter{figure}{0} 
\setcounter{equation}{0} 
\setcounter{table}{0} 
\section {Multi-Similarity \& Contrastive Learning}
\subsection{Multi-Similarities in Contrastive Learning}
Two distinct types of similarities, as illustrated in Appendix Figure \ref{fig:similarities}, can be identified: \textit{self-similarity} (the pairwise similarity between two objects, typically defined through cosine similarity) and \textit{relative similarity} (distinctions in self-similarity with other pairs) \cite{wang2019multi}.
\begin{figure}[H] 
    \centering
    \includegraphics[width=0.48\textwidth]{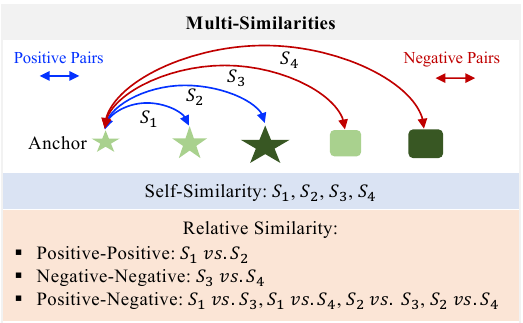}
    \caption{Illustration of Different Types of Similarities.}
    \label{fig:similarities}
\end{figure}

\subsection{Current Molecular Graph Contrastive Learning Approaches}
In current molecular graph contrastive learning approaches, positive pairs are commonly formed through either \textit{data augmentation} \cite{sun2021mocl, you2020graph}, employing techniques such as node deletion, edge perturbation, subgraph extraction, attribute masking, and subgraph substitution, or \textit{domain knowledge}, as demonstrated by reactant-product pairing \cite{wang2022chemicalreactionaware} or conformer grouping \cite{moon20233d}.

\begin{figure}[H] 
    \centering
    \includegraphics[width=0.9\textwidth]{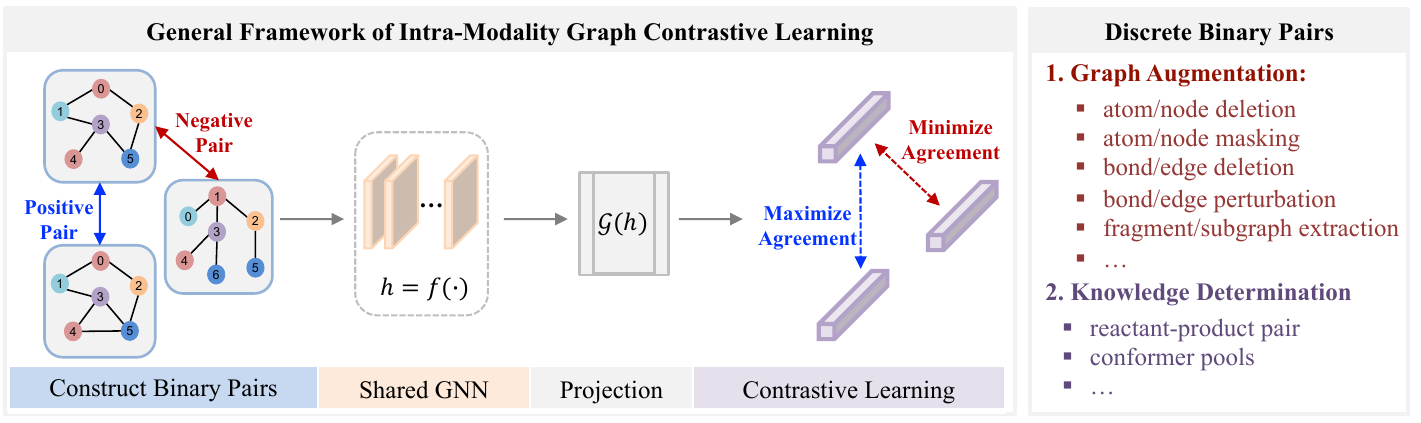}
    \caption{General framework of Intra-Modality Graph Contrastive Learning. It relies on definition of positive and negative pairs.}
    \label{fig:traditional-cl}
\end{figure}

\section{Supplementary Proof}
\subsection{Revisiting Theorem of Convergent Similarity Learning}
\label{appendix:gml-guide-proof}
Let $\mathcal{S}$ be a set of instances with size $|\mathcal{S}|$, and let $\mathcal{P}$ represent the tunable latent representations of instances in $\mathcal{S}$ such that $|\mathcal{P}| = |\mathcal{S}|$. For any two instances $i, j \in \mathcal{S}$, their latent representations are denoted by $\mathcal{P}_{i}$ and $\mathcal{P}_{j}$, respectively. Let $t_{i,j}$ represent the target similarity between instances $i$ and $j$ in a given domain, and $d_{i,j}$ be the similarity between $\mathcal{P}_i$ and $\mathcal{P}_j$ in the latent space.
\begin{theorem}[Theorem of Convergent Similarity learning]
Given $t_{i,j}$ is non-negative and $\{t_{i,j}\}$ satisfies the constraint $\sum_{j=1}^{|\mathcal{S}|}t_{i,j} = 1$, consider the loss function for an instance $i$ defined as follows:
\begin{equation}
    L(i) = -\sum_{j=1}^{|\mathcal{S}|} t_{i,j} \log \left( \frac{e^{d_{i,j}}}{\sum_{k=1}^{|\mathcal{S}|} e^{d_{i,k}}} \right)
\end{equation}
then when it reaches ideal optimum, the relationship between $t_{i,j}$ and $d_{i,j}$ satisfies:
\begin{equation}
    \text{softmax}(d_{i,j}) = t_{i,j}
\end{equation}
\end{theorem}

\begin{proof}
In order to optimize the loss $L(i)$, we need to set the following partial derivative to be 0 for each $d_{i,j}$ with $1\leq j \leq |\mathcal{M}|$. Here are the detailed steps:
\begin{align*}
\frac{\partial L(i)}{\partial {d_{i,j}}} 
&= \frac{\partial}{\partial d_{i,j}}\underbrace{\left( - t_{i,j} \log \frac{e^{d_{i,j}}}{e^{d_{i,j}} + \sum_{\substack{k \neq j}} e^{d_{i,k}}} \right)}_{\text{When the numerator includes } e^{d_{i,j}}} + \frac{\partial}{\partial d_{i,j}}\underbrace{\left( \sum_{\substack{k \neq j}} - t_{i,k} \log \frac{e^{d_{i,k}}}{e^{d_{i,j}} + \sum_{\substack{k \neq j}} e^{d_{i,k}}} \right)}_{\text{When the numerator does not include }e^{d_{i,j}}} \\
&= -(t_{i,j} - t_{i,j} \cdot \text{softmax}(d_{i,j})) - \sum_{\substack{k \neq j}} t_{i,k} \cdot \text{softmax}(d_{i,j}) \\
&= - \left( t_{i,j} - \left(t_{i,j} + \sum_{\substack{k \neq j}} t_{i,k} \right) \cdot \text{softmax}(d_{i,j})\right)
\end{align*}
Since $\sum_{l=1}^{|\mathcal{M}|}t_{i,l} = 1$, we can further simplify it as 
\begin{align*}
\frac{\partial L(i)}{\partial {d_{i,j}}} = - (t_{i,j} -  \text{softmax}(d_{i,j}))
\end{align*}
In order to optimize, we need to see the above partial derivative to be 0:
\begin{align*}
\frac{\partial L(i)}{\partial {d_{i,j}}} = - (t_{i,j} -  \text{softmax}(d_{i,j})) = 0
\end{align*}
In addition, the corresponding second partial derivative denoted as $\frac{\partial L(i)}{\partial {d^2_{i,j}}}$ manifests as follows:
\begin{align*}
\frac{\partial L(i)}{\partial {d^2_{i,j}}} = \text{softmax}(d_{i,j})(1- \text{softmax}(d_{i,j}))
\end{align*}
As $\text{softmax}(d_{i,j})$ takes values within the open interval (0,1), it follows that $\frac{\partial L(i)}{\partial {d^2_{i,j}}}$ is always positive. Consequently, the global optimum is global minimum.\\
Furthermore, when it comes to optimum:
\begin{align*}
t_{i,j} &= \text{softmax}(d_{i,j}) \\
d_{i,j} &= \log(t_{i,j}) + \log \left( \sum_{\substack{1\leq l \leq |\mathcal{M}|}} e^{d_{i,j}} \right)
\end{align*}
It is easy to show that when it reaches optimum, $d_{i,j}$ is consistent with target similarity metric $t_{i,j}$. Without loss of generosity, suppose $t_{i,j} > t_{i,j'}$ :
\begin{align*}
d_{i,j} - d_{i,j'} &= \log(t_{i,j}) + \log \left( \sum_{\substack{1\leq l \leq |\mathcal{M}|}} e^{d_{il}} \right) - \left( \log(t_{i,j'}) + \log \left( \sum_{\substack{1\leq l \leq |\mathcal{M}|}} e^{d_{il}} \right) \right) \\
&= \log(t_{i,j}) - \log(t_{i,j'}) \\
&= \log\left( \frac{t_{i,j}}{t_{i,j'}} \right) > 0
\end{align*}
\end{proof}

\subsection{Guarantee of Sum of Fused Multimodal Similarity}
\label{appendix:guarantee-fusion-sum}
Given sets of uni-modal generalized similarity $\{t^{R}\}$ and $\sum w_{t^{R}} = 1$, the sum of fused multimodal similarity also equals 1, as demonstrated below:
\begin{align*}
\sum (t_{i,j}^{R})
& = \sum \sum(w_{R} \cdot t^{R}_{i,j})\\
& = \sum (w_{R}  \sum t^{R}_{i,j})\\
& = \sum w_{R} \cdot 1 = 1
\label{equ:graph-guide-zhou-sum}
\end{align*}

\section{Revisiting Multi-Similarity Settings}
\subsection{Encoders \& Packages}
To derive the self-similarities, we need to reply on pre-trained encoders or well-defined packages as follows:

\begin{table}[h]
\begin{center}
\begin{small}
\caption{Encoders and packages used to produce self-similarities}
\begin{tabular}{l|c|c|c}
\hline
Unimodal & Representation & Encoder/Package & Pre-trained Source \\ 
\hline
Image & 2D image & CNN & Img2mol \cite{clevert2021img2mol} \\
SMILES & Sequence &Transformer  & CReSS \cite{yang2021cross}  \\
\textsuperscript{13}CNMR Spectrum& Sequence & 1D CNN & AutoEncoder \cite{costanti2023deep}  \\
\textsuperscript{13}CNMR peak & Scalar & NMRShiftDB2 \cite{steinbeck2003nmrshiftdb} & N/A  \\
Fingerprint & Sequence & RDKit \cite{landrum2006rdkit}  & N/A  \\

\hline

\end{tabular}
\end{small}
\end{center}
\end{table}

\subsection{Self-Similarity at Graph Level}
\label{appendix:uni-modal-self}
\textbf{Fingerprint.} The mathematical formula of fingerprint similarity, denoted as $S_{i,j}^{F}$, can be viewed as follows:
\begin{align}
   S_{i,j}^{F} & = Tanimoto(A, B) = \frac{|A \cap B|}{|A \cup B|}  
\end{align}
where \( A \) and \( B \) are sets of molecular fragments for molecule $i$ and $j$, and \( |A \cap B| \) and \( |A \cup B| \) denote the size of their intersection and union, respectively.

\textbf{Image.} The self-similarity for Image, denoted as $S_{i,j}^{I}$, can be defined as follows:
\begin{equation}
    S_{i,j}^{I} = Cos( \mathcal V_{i}, \mathcal V_{j}) = \frac{\mathcal{V}_i \cdot \mathcal{V}_j^T}{\|\mathcal{V}_i\| \cdot \|\mathcal{V}_j\|}
\end{equation}
where $\mathcal V_{i},  \mathcal V_{j} $ represents the embedding of Image for two given molecules.

\textbf{NMR Spectrum.} The self-similarity for NMR spectrum, denoted as $S_{i,j}^{C}$, can be defined as follows:
\begin{equation}
    S_{i,j}^{C} = Cos( \mathcal V_{i}, \mathcal V_{j}) = \frac{\mathcal{V}_i \cdot \mathcal{V}_j^T}{\|\mathcal{V}_i\| \cdot \|\mathcal{V}_j\|}
\end{equation}
where $\mathcal V_{i},  \mathcal V_{j} $ represents the embedding of NMR spectra for two given molecules.

\textbf{Smiles.} The self-similarity for Smiles, denoted as $S_{i,j}^{S}$, can be defined as follows:
\begin{equation}
    S_{i,j}^{S} = Cos( \mathcal V_{i}, \mathcal V_{j}) = \frac{\mathcal{V}_i \cdot \mathcal{V}_j^T}{\|\mathcal{V}_i\| \cdot \|\mathcal{V}_j\|}
\end{equation}
where $\mathcal V_{i},  \mathcal V_{j} $ represents the embedding of Smiles for two given molecules.

\subsection{A Brief Introduction to PPM}
\label{appendix:ppm}
In chemistry, $^{13}$C NMR stands out as a common technique for structural analysis by revealing molecular structures by elucidating the chemical environments of carbon atoms and their magnetic responses to external fields \cite{gerothanassis2002nuclear,lambert2019nuclear}. It quantifies these features in parts per million (ppm) relative to a reference compound, such as tetramethylsilane (TMS), thereby simplifying comparisons across experiments. As a result, the continuous peak positions, measured in parts per million (ppm), offer a robust knowledge span—a natural ordering metric that can be employed to derive measures of similarity \cite{xu2023molecular}.

\subsection{Configuration of Fused Multimodal Generalized Similarity Metric}
\label{appendix:config-fusion}
A simple linear combination is used to formulate the multimodal multi-similarity $t_{i,j}^{M}$ between the $i^{th}$ and $j^{th}$ molecules, represented as a graph-level similarity $t_{i,j}^{g}$, as follows:
\begin{align}
t_{i,j}^{g} & = t_{i,j}^{M} = w_{SM} \cdot t^{SM}_{i,j} + w_{{C}} \cdot t^{C}_{i,j} + w_{{I}} \cdot t^{I}_{i,j} + w_{{F}} \cdot t^{F}_{i,j}
\label{equ:graph-guide-zhou}
\end{align}
where $t^{SM}_{i,j}$ denotes the similarity based on SMILES, $t^{C}_{i,j}$ denotes the similarity with respect to $^{13}$C NMR spectrum, $t^{I}_{i,j}$ denotes the similarity regarding images, and $f$ denotes the similarity based on fingerprints, $w_{SM}$, $w_{C}$, $w_{I}$, and $w_{F}$ are the pre-defined weights for their respective similarity, and $w_{SM} + w_{{C}} + w_{{I}} + w_{{F}} = 1$.

For pre-defined weights, denoted as $w_{SM}$, $w_{C}$, $w_{I}$, and $w_{F}$, we configure them with various settings for ablation studies, as shown in the following table:
\begin{table}[ht]
    \begin{center}
    \begin{small}
    \caption{The configuration of weight for each unimodality similarity in different Fusion. In particular $w_{S}$, $w_{N}$, $w_{M}$ and $w_{F}$ represents the weight of Smiles, NMR, Image and Fingerprint, respectively.}
    \label{tab:appendix-config_fusion}
    \vspace{8pt}
    \begin{sc}
        \begin{tabular}{lcccc}
        \toprule
        Fused multimodal& $w_{SM}$  & $w_{N}$ & $w_{M}$ & $w_{F}$  \\
        \hline
        Smiles & 1.00 & 0.00 & 0.00 & 0.00\\

        NMR & 0.00 & 1.00 & 0.00 & 0.00\\

        Image & 0.00 & 0.00 & 1.00 & 0.00\\ 

        Fingperprint & 0.00 & 0.00 & 0.00 & 1.00\\
        
        $\text{Fusion}_{\text{Smiles}}$ & 0.70 & 0.10 & 0.10 & 0.10 \\

        $\text{Fusion}_{\text{NMR}}$ & 0.10 & 0.70 & 0.10 & 0.10\\
        
        $\text{Fusion}_{\text{Image}}$ &  0.10 & 0.10 & 0.70 & 0.10 \\
        
        $\text{Fusion}_{\text{Fingerprint}}$&  0.10 & 0.10 & 0.10 & 0.70 \\
        
        $\text{Fusion}_{\text{Average}}$ &  0.25 & 0.25 & 0.25 & 0.25\\
        \bottomrule
        \end{tabular}
    \end{sc}
    \end{small}
    \end{center}

\vskip -0.1in
\end{table}

\section{Experimental Settings}
\label{appendix:exp-setting}
\subsection{Pre-Training Setting}
During pretraining, we utilized an Adam optimizer with a learning rate set to 0.001, spanning 200 epochs and employing a batch size of 256. The model was trained on 30,000 data points. The NMR data were experimental data, extracted from NMRShiftDB2 \cite{steinbeck2003nmrshiftdb}. Other chemical modalities, such as images, fingerprints and graphs, were produced from SMILES by RDKit \cite{landrum2006rdkit}.


\subsection{Fine-Tuning Setting}

\subsubsection{Datasets} 

For fine-tuning, our model was trained on 11 drug discovery-related benchmarks sourced from MoleculeNet \cite{wu2018moleculenet}. Eight of these benchmarks were designated for classification downstream tasks, including BBBP, BACE, SIDER, CLINTOX, HIV, MUV, TOX21, and ToxCast, while three were allocated for regression tasks, namely ESOL, Freesolv, and Lipo. The datasets were divided into train/validation/test sets using a ratio of 80\%:10\%:10\%, accomplished through the scaffold splitter \cite{halgren1996merck, landrum2006rdkit} from Chemprop \cite{yang2019analyzing, heid2023chemprop}, like previous works. The scaffold splitter categorizes molecular data based on substructures, ensuring diverse structures in each set. Molecules are partitioned into bins, with those exceeding half of the test set size assigned to training, promoting scaffold diversity in validation and test sets. Remaining bins are randomly allocated until reaching the desired set sizes, creating multiple scaffold splits for comprehensive evaluation.


\subsubsection{Baselines}
\label{appendix:baselines}
We systematically compared GraphMSL's performance with various state-of-the-art baseline models across different categories. In the realm of supervised models, AttentiveFP \cite{xiong2019pushing} and DMPNN \cite{yang2019analyzing} stand out by leveraging graph attention networks and node-edge interactive message passing, respectively. The unsupervised learning method N-Gram \cite{liu2019n} employs graph embeddings and short walks for graph representation. Predictive self-supervised learning methods, such as GEM \cite{fang2022geometry} and Uni-Mol \cite{Zhou2023UniMolAU}, are specifically designed for predicting molecular geometric information. GROVER \cite{rong2020self} integrates Message Passing Networks into a Transformer-style architecture, creating a class of more expressive encoders for molecules. Moreover, our evaluation encompasses a range of contrastive learning methods, namely InfoGraph \cite{sun2019infograph}, GraphCL \cite{you2020graph}, MolCLR \cite{wang2022molecular}, and GraphMVP \cite{liu2022pretraining}, all serving as essential baselines. The baseline results are collected from recent works \cite{fang2022geometry, Zhou2023UniMolAU, moon20233d, fang2023knowledge}.

\subsubsection{Evaluation} 
To assess the effectiveness of our fine-tuned model, we measure the ROC-AUC for classification downstream tasks, and the root mean squared error (RMSE) metric for regression tasks. In order to ensure a fair and robust comparisons, we conduct three independent runs using three different random seeds for scaffold splitting across all datasets. The reported performance metrics are then averaged across these runs, and the standard deviation is computed as prior works.

\section{More Ablation Study}
$\text{GraphMSL}_{\text{bi-level}}$ demonstrates top performance in BBBP, BACE, CLINTOX, HIV, MUV, and Tox21. Additionally, within the $\text{GraphMSL}{\text{bi-level}}$ framework, $\text{Fusion}{\text{Smiles}}$ enhances performance in Sider, CLINTOX, and MUV. $\text{Fusion}{\text{NMR}}$ improves results in BBBP and HIV, while $\text{Fusion}_{\text{Image}}$ contributes to better outcomes in the Tox21 and ToxCast tasks.

\begin{table*}[h]
\caption{Ablation study on the performances of GraphMSL with bi-level (Graph + Node) and solo-node-level. The best results are denoted in boldf, and the second-best are indicated with underlining.}
\label{tab:multimodal-metrics-bi-level}
\begin{center}
\begin{tiny}
\begin{sc}
\begin{tabular}{lcccccccc|ccc}
\toprule
Data Set & BBBP & bace & Sider & Clintox & HIV & MUV & Tox21 & ToxCast & ESOL & FreeSolv & Lipo \\
\midrule
SMILES + Node & 93.7$\pm$1.3 & 93.2$\pm$3.7 & 65.9$\pm$1.7 & 87.7$\pm$7.8 & 82.3$\pm$1.8 & 80.9$\pm$5.5 & 84.7$\pm$1.7 & 70.4$\pm$1.5 & 0.873$\pm$ 0.085 & 1.658$\pm$ 0.243 & 0.594$\pm$ 0.031\\
NMR + Node & 91.9$\pm$3.1 & 92.8$\pm$1.6 & 65.7$\pm$1.6 & 89.5$\pm$3.4 & 81.2$\pm$1.2 & 80.6$\pm$5.5 & 85.1$\pm$0.2 & 69.3$\pm$0.1 & 1.052$\pm$ 0.105 & 2.391$\pm$ 0.175 & 0.654$\pm$ 0.025 \\
Image + Node & 94.1$\pm$1.7 & 90.8$\pm$0.7 & 63.8$\pm$2.8 & 86.5$\pm$8.0 & 80.3$\pm$1.5 & 76.7$\pm$2.7 & 85.3$\pm$1.1 & 70.8$\pm$1.5 &\textbf{0.843$\pm$ 0.094} & \textbf{1.601$\pm$ 0.057} & \textbf{0.562$\pm$ 0.005} \\
Fingerprint + Node & 90.2$\pm$8.4 & \textbf{94.5$\pm$0.7} & 64.3$\pm$2.9 & 91.0$\pm$1.5 & 82.0$\pm$2.4 & 79.2$\pm$5.9 & 85.7$\pm$0.5 & 69.7$\pm$1.3 &1.170$\pm$ 0.174 & 2.801$\pm$ 0.276 & 0.607$\pm$ 0.034\\
\hline


$\text{Fusion}_{\text{Smiles}}$ + Node & 91.7$\pm$5.3 & 91.5$\pm$1.7 & \textbf{67.3$\pm$0.6} & \textbf{93.8$\pm$0.8} & 82.1$\pm$1.7 & \textbf{81.5$\pm$3.7} & 85.1$\pm$0.1 & 70.4$\pm$1.3 & 0.965$\pm$0.085 & 2.859$\pm$0.281 & 0.647$\pm$0.029 \\

$\text{Fusion}_{\text{NMR}}$ + Node & \textbf{94.3$\pm$0.8} & 93.6$\pm$1.9 & 66.8$\pm$1.4 & 90.4$\pm$3.1 & \textbf{83.0$\pm$0.7} & 80.1$\pm$3.5 & 85.5$\pm$0.6 & 70.6$\pm$1.8 & 1.009$\pm$0.160 & 2.224$\pm$0.368 & 0.589$\pm$0.038\\

$\text{Fusion}_{\text{Image}}$ + Node & 94.2$\pm$1.2 & 93.1$\pm$2.5 & 66.4$\pm$1.6 & 90.7$\pm$3.5 & 82.0$\pm$2.4 & 80.8$\pm$3.8 & \textbf{86.1$\pm$0.8} & \textbf{71.2$\pm$1.1} & 0.898$\pm$0.098 & 1.691$\pm$0.386 & 0.579$\pm$0.018 \\

$\text{Fusion}_{\text{Fingerprint}}$ + Node& 91.6$\pm$5.0 & 94.3$\pm$2.4 & 66.4$\pm$1.9 & 85.3$\pm$6.8 & 82.0$\pm$2.4 & 80.6$\pm$3.2 & 85.2$\pm$0.2 & 69.8$\pm$1.1 & 1.037$\pm$0.170 & 2.093$\pm$0.090 & 0.607$\pm$0.034 \\

$\text{Fusion}_{\text{Average}}$ + Node & 92.7$\pm$1.5 & 92.6$\pm$2.1 & 65.6$\pm$0.7 & 89.3$\pm$4.0 & 81.8$\pm$1.7 & 81.0$\pm$5.0 & 85.4$\pm$1.3 & \underline{71.2$\pm$1.9} & 1.019$\pm$0.118 & 1.733$\pm$0.267 & 0.593$\pm$0.004\\
\hline
Node & 93.4$\pm$2.7 & 89.3$\pm$1.7 & 62.8$\pm$2.1 & 86.1$\pm$5.4& 82.1$\pm$0.4 & 75.4$\pm$5.2 & 84.9$\pm$1.0 & 
70.6$\pm$0.8 & 0.924$\pm$0.083 &1.707$\pm$0.126 & 0.587$\pm$0.021  \\
\bottomrule
\end{tabular}
\end{sc}
\end{tiny}
\end{center}
\end{table*}



\end{document}